\documentclass{article} 
\usepackage[final]{colm2025_conference}

\usepackage{microtype}
\usepackage{hyperref}
\usepackage{url}
\usepackage{amsmath, amsthm, amssymb, amsfonts}
\usepackage{booktabs}
\usepackage{graphicx}
\usepackage{tcolorbox}

\newtheorem{theorem}{Theorem}[section]
\newtheorem{lemma}[theorem]{Lemma}
\newtheorem{definition}[theorem]{Definition}
\newtheorem{proposition}[theorem]{Proposition}
\newtheorem{assumption}[theorem]{Assumption}

\usepackage{lineno}

\definecolor{darkblue}{rgb}{0, 0, 0.5}
\hypersetup{colorlinks=true, citecolor=darkblue, linkcolor=darkblue, urlcolor=darkblue}

\title{Towards Infinite Length Extrapolation : A Unified Approach}

\author{Nitin Vetcha \\
Department of Computer Science\\
Indian Institute of Science, Bangalore \\
\texttt{nitinvetcha@iisc.ac.in} \\
}

\begin{document}

\ifcolmsubmission
\linenumbers
\fi

\maketitle

\begin{abstract}
Large language models (LLMs) have revolutionized natural language processing, but their ability to process long sequences is fundamentally limited by the context window size during training. Existing length extrapolation methods often suffer from performance degradation or computational inefficiencies. We thereby use a unified framework that reinterprets positional encoding methods as a decomposition of the attention score into a multiplicative transformation and an additive bias. This perspective not only subsumes popular approaches such as relative position embeddings and attention-bias moderated approaches but also exposes their inherent limitations in handling long-range dependencies. To address these shortcomings, motivated by our framework, we introduce Adaptive Positional Encoding (APE), which leverages adaptive frequency modulation and an intricately designed decay bias that incorporates linear, logarithmic, and square-root terms. Our theoretical analysis establishes conditions for infinite-context extrapolation, ensuring that the softmax normalization remains well-defined over unbounded sequences while preserving long-distance correlations, entropy boundedness and gradient positional sensitivity. We substantiate our claims with an experimental case study on TinyStories dataset as well as a new synthetic dataset, \emph{Long Tiny Stories} featuring stories up to 32,000 words. Relevant code, dataset and model weights are available at \url{https://anonymous.4open.science/r/Check-2DAD/}.
\end{abstract}

\section{Introduction}

The remarkable contemporary progress in natural language processing has been due to rapid advances in LLM technologies based on the transformer architecture (\cite{vaswani2017attention}). A critical factor influencing the performance of LLMs is their scale and context-window sizes, as robust natural language processing requires effective context retention over long sequences (\cite{pawar2024and}) -- LLMs with longer context windows show substantial improvements in tasks requiring processing extended language patterns, for example, in long document summarization, question-answering on lengthy documents, to name a few (\cite{qian2024long}). Increasing adoption of LLMs for processing long texts across applications in both the scientific enterprise (e.g., LLMs for paper analysis for scientific workflows), and industry (e.g., LLMs for code review in large code bases), makes handling long-context handling a core issue in advancing LLM technologies (\cite{watkins2024guidance,lu2023llama}). While various techniques have been proposed to mitigate this constraint (e.g., relative position encodings and sparse attention mechanisms), there lacks a unified approach and a mathematical formalization of long context understanding for LLMs. In this regard, main contributions of our work include

\begin{itemize}
  \item Development of a unified theoretical framework that decomposes the attention score into a multiplicative transformation and an additive bias, showing that existing positional encoding methods fit within this general approach, and use this information to propose several novel approaches for infinite-context extrapolation, accompanied by rigorous theoretical analysis for unbounded token sequences.
  \item In depth, mathematical analysis and empirical case study of one such motivated novel method, the Adaptive Positional Encodings which yield superior empirical performance over state-of-the-art techniques, improving metrics such as perplexity and reducing memory footprints at train time.
  \item Introduction of a new dataset tailored for evaluating long-context processing abilities of small large language models, addressing current limitations in evaluation procedures.
\end{itemize}

\section{Related Work}

\textbf{Position Encoding Methods} - Various context extension techniques have been proposed in recent literature, primarily based on positional embeddings due to their impact on length generalization (\cite{kazemnejad2023impact}). Relative position embeddings have been shown to be very promising for a long time (\cite{shaw2018self}), with state-of-the-art being Rotary Position Embeddings (\cite{su2024roformer}) , that has thus far, prevailed and been studied extensively as the most effective technique across a variety of LLM architectures, benchmarks, tasks, and datasets (\cite{chen2023extending, press2022trainshorttestlong}) -- RoPE's effectiveness can be attributed to its capacity to overcome decaying relevance as the context-size increases, by instead modeling relative dependencies among token positions, additionally \textit{baking this information directly into the attention mechanism} through rotation matrices over the embedding dimensions of the query and key vectors (\cite{zhong2025understanding}). Although this technique helps with training with large context windows, crucially, the core mechanism in RoPE and it's extensions also allows \textit{inference time} extensions through interpolation to beyond context-sizes seen at training time. Another approach is through the induction of bias terms in attention mechanism such as T5-Bias and Attention With Linear Biases (\cite{press2022train}). Additionally, \cite{zhao2023length} conduct a survey on length Extrapolation of Transformers from the perspective of positional encodings.

\textbf{Sparse and Memory-based Approaches} - Compressive memory systems represent a promising direction for extreme sequence lengths. Recent methods include, Google's Infini-attention (\cite{munkhdalai2024leavecontextbehindefficient}) which exemplifies this approach by integrating compressive memory into vanilla dot-product attention layers. This technique enables transformers to process theoretically infinite input sequences while maintaining bounded memory footprint and computation requirements. Another approach is SPARSEK Attention (\cite{lou2024sparserfastermoreefficient}) which introduces a scoring network and a differentiable top-k mask operator to select a constant number of key-value pairs for each query. It achieves linear time complexity and constant memory footprint during generation, significantly improving speed and reducing computational overhead
There have been several prior techniques, as well, such as Longformer (\cite{beltagy2020longformer}) and BigBird (\cite{zaheer2020big}) reduce the computational burden by sparsifying the attention mechanism and using low-rank approximations (\cite{chen2024longloraefficientfinetuninglongcontext}). These methods often discard distant tokens or employ fixed attention windows, which can limit their ability to capture long-range dependencies.
Memory-based approaches like Transformer-XL (\cite{dai2019transformer}) and Compressive Transformers (\cite{rae2019compressive}) introduce recurrence or memory modules to recycle past context. Although effective, they typically require architectural modifications and additional training.

In addition to the above two broad classes, there are several other recent directions as well such as retrieval based approaches (\cite{xu2024retrieval}, \cite{mohtashami2023randomaccess}), multi-agent collaboration (\cite{zhang2024chainagentslargelanguage} introduce a framework where multiple LLMs collaborate to process long-context tasks by dividing inputs into chunks and assigning them to agents for sequential processing) and reasoning based methods (\cite{qian2024long} propose a divide-and-conquer strategy for short LLMs to access and utilize long-context information effectively). In addition, there is an increasing interest in innovative architectures like Cascading KV Cache () which aim to rethink how LLMs process vast (\cite{sun2024shadowkv}, \cite{willette2024training}, \cite{li2024scbench}) inputs by weaving broader memory tapestries without retraining. These methods promise human-like context handling while reducing inference latency and compute overhead.

\section{Mathematical Analysis}

\subsection{Background}
We use the \emph{Generalized Positional Encoding (GPE)} framework, which unifies several existing methods by decomposing the modification of attention scores into a \emph{multiplicative transformation} and an \emph{additive bias} based on the relative position \( n = i - j \) of tokens. For further analysis, we consider Rotary Position Embeddings (RoPE) and Attention with Linear Biases (ALiBi) as representatives of the widely prevalent approaches - relative position embeddings and attention-bias methods. 

Let \( q_i, k_j \in \mathbb{R}^d \) be the query and key vectors corresponding to positions \( i \) and \( j \) respectively, and denote by \( n = i - j \) their relative position. The GPE framework modifies the raw attention score via
\begin{equation}\label{eq:GPE}
A(n) = f(n) \cdot \Bigl( q_i^\top \mathbf{W}(n) \, k_j \Bigr) + b(n),
\end{equation}
where:
\( \mathbf{W}(n) \in \mathbb{R}^{d \times d} \) is a \emph{position-dependent transformation matrix} (for example, a rotation or scaling), \( f : \mathbb{Z} \to \mathbb{R}^+ \) is a \emph{gain/decay function} that modulates the multiplicative term and \( b : \mathbb{Z} \to \mathbb{R} \) is an \emph{additive bias} function.

In this framework, one can recover   \textit{RoPE} ( by setting $
\mathbf{W}(n) = \mathbf{R}(n), \quad f(n)=1, \quad b(n)=0,
$ where \(\mathbf{R}(n)\) is a block-diagonal rotation matrix with angles $    \theta_m = \text{base}^{-2m/d} \cdot n,$ for block index \( m \) ) and  \textit{ALiBi} (by setting
    $
    \mathbf{W}(n) = \mathbf{I}, \quad f(n)=1, \quad b(n)=-m\cdot|n|,
    $ with a positive slope parameter \( m > 0 \)).

We now first mathematically state, three key properties which are essential for infinite-context extrapolation and then present the intuition for it.

\begin{definition}[Convergent Normalization]
A positional encoding method is said to achieve infinite-context extrapolation if its softmax normalization denominator converges even for unbounded sequences. Formally, letting
$
Z = \sum_{n=0}^{L} e^{A(n)},
$ the method exhibits convergent normalization if
$
\lim_{L\to\infty} Z < \infty
$.
\end{definition}

\begin{definition}[Entropy Boundedness]
A positional encoding method is said to exhibit entropy boundedness if the Shannon entropy of its resulting attention distribution is finite. That is, if the normalized attention weights are given by
$
p(n) = {e^{A(n)}}/{\sum_{n\in\mathbb{Z}} e^{A(n)}},
$
then entropy boundedness is the property that
$
H = -\sum_{n\in\mathbb{Z}} p(n)\log p(n) < \infty.
$
\end{definition}

\begin{definition}[Long-Distance Correlation Preservation (LDCP)]
A positional encoding method satisfies \emph{LDCP} if there exists a constant \(C >0\) and an infinite set \(S\subset\mathbb{Z}\) such that for all \(n\in S\), 
$
\left|\mathbb{E}[A(n)]\right| \ge C,
$
or, more generally, if with high probability the absolute attention score does not vanish as \(|n|\) increases.
\end{definition}

\begin{definition}[Gradient Positional Sensitivity (GPS)]
A positional encoding method is said to exhibit gradient positional sensitivity if the gradient of the attention score with respect to the query (key) vector depends explicitly on the relative position \(n\). Formally, there exists a non-constant function \(g(n)\) such that
\[
\frac{\partial A(n)}{\partial q} = g(n)\cdot \frac{\partial (q^\top k)}{\partial q}.
\]
\end{definition}

\emph{Convergent normalization} guarantees that the softmax denominator remains finite as the sequence length tends to infinity, ensuring that the attention distribution is well-defined. \emph{Entropy boundedness} is necessary to prevent the distribution from becoming too diffuse or overly concentrated. This balance is crucial for stable learning and inference. Lastly, \emph{Long-Distance Correlation Preservation (LDCP)} makes sure that even tokens that are far apart still contribute meaningfully to the attention scores, allowing the model to capture long-range dependencies. While there could be possibly more properties, we present our further analysis in the backdrop of convergent normalization, entropy boundedness and LDCP. \emph{Gradient Positional Sensitivity (GPS)} property ensures that the backpropagated gradient carries positional information, thereby allowing the learning dynamics to adapt based on \(n\).

\begin{theorem}[Equivalence of Entropy Boundedness and Convergent Normalization]\label{thm:entropy_extrapolation} .
\begin{itemize}
  \item[(i)] If $Z<\infty$, then $H < \infty$.
  \item[(ii)] Conversely, if for every $L$ the truncated distribution
  $
  p_L(n)={e^{A(n)}}/{\sum_{k=0}^L e^{A(k)}}
  $
  has entropy uniformly bounded by $H_{\max}$ (i.e. $H_L \le H_{\max}$ for all $L$), then $Z<\infty$. \hfill (Proof - \ref{ref:proof_entropy_equivalence} )
\end{itemize}
\end{theorem}

\begin{lemma}[LDCP - Convergent Normalization Tradeoff ]\label{thm:ldcp_div}
Assume a positional encoding method satisfies LDCP, then it cannot simultaneously satisfy convergent normalization.
\end{lemma}
\begin{proof}
For each \(n\in S\), by Jensen's inequality, $
\mathbb{E}[e^{A(n)}]\ge e^{\mathbb{E}[A(n)]}$ and by LDCP $e^{\mathbb{E}[A(n)]} \ge e^{C}.$ Since \(S\) is infinite, the sum $
\sum_{n\in S} \mathbb{E}[e^{A(n)}]\ge \sum_{n\in S} e^{C}=\infty,
$ diverges. Thus, convergent normalization is violated because $S \subset \mathbb{Z}$. 
\end{proof}

\subsection{Limitations of Existing Methods}

\begin{proposition}
    For any query and key vectors with bounded norms, RoPE violates Convergent Normalization and Entropy Boundedness. \hfill (Proof - \ref{ref:proof_RoPE_limitations})
\end{proposition}

\begin{assumption} \label{ref:assumption}
The query and key vector  \(q,k\) are independent uniformly distributed unit vectors on the sphere \(\mathbb{S}^{d-1}\) in \(\mathbb{R}^d\). Then,
$
\mathbb{E}[q^\top k]=0 \,\text{and}\,\text{Var}[q^\top k]=1/d.
$
\end{assumption}

\begin{proposition}
ALiBi fails to satisfy LDCP under assumption \ref{ref:assumption} and GPS. \hfill (Proof - \ref{ref:proof_ALiBi_limitation})
\end{proposition}

\subsection{Advantages of Existing Methods}

\begin{proposition}
RoPE satisfies LDCP under assumption \ref{ref:assumption} and GPS. \hfill (Proof - \ref{ref:proof_RoPE_Advantage})
\end{proposition}

\begin{proposition}
For any query and key vector with bounded norms, ALiBi is able to achieve convergent normalization and entropy boundedness. \hfill (Proof -\ref{ref:proof_ALiBi_Advantage})
\end{proposition}

In order to balance the limitations and advantages of both approaches, we propose Adaptive Positional Encoding (APE) in the next section.

\section{Proposed Methodology}
\begin{definition}[Adaptive Positional Encoding (APE)]
\begin{equation}\label{eq:APE}
A_{APE}(n) = \underbrace{\text{temp}(n) \cdot \Bigl( q_i^\top \mathbf{R}_{\alpha(n)}(n) \, k_j \Bigr)}_{\text{Adaptive Multiplicative}} + \underbrace{b(n)}_{\text{Adaptive Additive}},
\end{equation}
where the components are specified as follows:
\begin{enumerate}
    \item \textbf{Frequency Adaptation:} The rotation matrix is given by 
    $
    \mathbf{R}_{\alpha(n)}(n) = \mathbf{R}\bigl(\theta(n)\bigr),
    $
    with an adaptive rotation angle 
    $
    \theta(n) = {n}/{\alpha(n)}, \quad\text{where} \quad \alpha(n) \propto \text{attention\_entropy}(n).
    $
    That is, for positions where the attention distribution is less peaked, a larger \(\alpha(n)\) is used to achieve a smoother (more gradual) decay.
    
    \item \textbf{Temperature Scheduling:} We set
    $
    \text{temp}(n) = 1/({1+\lambda|n|}), \quad \lambda>0,
    $
    which dampens scores for large \( |n| \).
    
    \item \textbf{Learnable Additive Bias:} We now define the bias term as
    $
    b(n) = -\delta|n| - \beta \log(1+|n|) - \gamma \sqrt{|n|},
    $
    with parameters \(\delta,\beta,\gamma>0\). This combination of linear, logarithmic, and square-root penalties is designed to penalize large distances sublinearly while still ensuring a sufficient decay.
\end{enumerate}
\end{definition}

\begin{proposition}
For any query and key vector with bounded norms, APE is able to achieve convergent normalization, entropy boundedness and gradient positional sensitivity.  \hfill (Proof - \ref{ref:proof_APE_advantage})
\end{proposition}

A noteworthy point is that, since bias $b(n)$ term of APE decays slower (due to the additional logarithmic and square-root terms) than the linear decay in ALiBi, the resulting attention distribution
$
p(n)={e^{A(n)}}/{Z}
$ has a heavier tail. This implies a larger expected value $\mathbb{E}_{p}[|n|]$ and, by the formula $
H = \log Z + \mathbb{E}_{p}[-A(n)],
$, this leads to a higher entropy relative to ALiBi, preserving useful long-range interactions. However, on the contrary, due to similar reasons as in ALiBi, APE as well fails to achieve global LDCP. We can however, show that APE performas better by relaxing the condition as follows,

\begin{definition}[Local LDCP Range]
The \emph{local LDCP range} \( N_{\text{LDCP}} \) is the largest integer such that for all \( |n| \leq N_{\text{LDCP}} \), the expected attention score satisfies
$
\left|\mathbb{E}[A(n)]\right| \geq C_1 > 0,
$
for a fixed threshold \( C_1 \).
\end{definition}

\begin{lemma}
$ N_{\text{LDCP}}^{\text{APE}} \geq N_{\text{LDCP}}^{\text{ALiBi}}$ for equivalent \( \delta = m \), due to sublinear terms under assumption \ref{ref:assumption}.
\end{lemma}

\begin{proof}
Solving \( |\mathbb{E}[A_{ALiBi}(n)]| = m|n| \geq C_1 \), yields $
|n| \leq {C_1}/{m}
$. Thus, $N_{\text{LDCP}}^{\text{ALiBi}} = \lfloor {C_1}/{m} \rfloor.$ For, $N_{\text{LDCP}}^{\text{ALiBi}}$, we need to solve \( |\mathbb{E}[A_{APE}(n)]| = \delta|n| + \beta \log(1+|n|) + \gamma\sqrt{|n|} \leq C_1 \).  
If \( \delta = m \), compare to ALiBi’s linear decay:
\[
\underbrace{m|n|}_{\text{ALiBi}} \leq \underbrace{m|n| + \beta \log(1+|n|) + \gamma\sqrt{|n|}}_{\text{APE}} \leq C_1.
\]
The inequality \( m|n| \leq C_1 \) gives \( |n| \leq {C_1}/{m} \), but APE’s additional terms allow larger \( |n| \) while still satisfying \( \mathbb{E}[A(n)] \geq C_1 \). Thus, \( N_{\text{LDCP}}^{\text{APE}} \geq N_{\text{LDCP}}^{\text{ALiBi}} \). 
\end{proof}
Now, as shown by the below lemma, even for values outside the local LDCP range, APE fairs better than ALiBi.

\begin{lemma}[APE’s Variance Advantage]\label{thm:ape_var}
For \( |n| > N_{\text{LDCP}}^{\text{ALiBi}} \), ALiBi’s attention scores are deterministically suppressed, while APE’s scores satisfy:
\[
\mathbb{P}\left(A_{\text{APE}}(n) \geq C_1\right) \geq \frac{\text{Var}[A_{\text{APE}}(n)]}{\left(\mathbb{E}[A_{\text{APE}}(n)] + C_1\right)^2 + \text{Var}[A_{\text{APE}}(n)]}.
\]
\end{lemma}

\begin{proof}
Using Chebyshev’s inequality for \( X = A_{\text{APE}}(n) \):
$
\mathbb{P}(X \geq \mathbb{E}[X] + t) \leq {\text{Var}[X]}/{t^2}.
$
Set \( t = C_1 - \mathbb{E}[X] \). For \( \mathbb{E}[X] \leq C_1 \):
\[
\mathbb{P}(X \geq C_1) \geq 1 - \frac{\text{Var}[X]}{\left(C_1 - \mathbb{E}[X] \right)^2}.
\]
\( \text{Var}[A_{\text{APE}}(n)] = {1}/{d(1+\lambda|n|)^2} \) ensures higher non-zero probability of \( A(n) \geq C_1 \) as  \( \text{Var}[A_{\text{ALiBi}}(n)] = {1}/{d} \)
\end{proof}

\section{Case Study}

In order to evaluate our approach and mathematical claims empirically, we conduct the following case study using our model, the Tiny Story Maker (TSM). The setup is in accordance with the Chinchilla Scaling laws as well, as the number of tokens for training is roughly in the order of 20 times the number of model parameters.

\subsection{Experimental Setup}

\textbf{Model Architecture:}  
TSM is a decoder-only 30M transformer similar to GPT-2 but in a scaled down version (number of heads = 6, number of layers = 6, embedding dimension = 384). The baselines RoPE and ALiBi are implemented following prior work i.e., for RoPE, the standard base of 10000 was used while for ALiBi, as recommended, the layer-wise slopes ($m$) is the geometric sequence which starts with $2^{-4/3}$ and uses the same value as the ratio. Our proposed APE method replaces the standard positional encoding module.

\textbf{Datasets:}  Owing to the proven effectiveness in training small LLMs, we use the Microsoft's TinyStories dataset \cite{eldan2023tinystories}. 
For controlled evaluations, we also create a synthetic dataset, namely LongTinyStories (LTS), in a very similar manner to how Tiny Stories has been generated (the only difference being is that at prompting instead of using the line - \emph{Write a short story (3-5 paragraphs)}, we replace it with \emph{Write a very long story (5-7 chapters)}). LTS contains stories containing a wide range of words ranging from 500 to 32,000 words. 

This was done because, each story in the validation split of Tiny Stories dataset, contains on average, around 180 words. Thus, evaluation of perplexity using this dataset, although may demonstrate the superiority of an approach in capturing short range dependencies when given very long prompts, it is not able to effectively demonstrated the medium or long-range dependencies as promised.

\begin{table}[ht]
\centering
\begin{tabular}{lccc}
\toprule
\textbf{Dataset} & \textbf{FRE} & \textbf{Gunning Fog} & \textbf{ARI} \\
\midrule
LongTinyStories & 93.06 & 3.63 & 2.64 \\
TinyStories & 105.19 & 4.83 & 0.85 \\
\bottomrule
\end{tabular}
\caption{Readability metrics across different datasets. Lower FRE and higher Gunning Fog and ARI scores is indicative of a more complex dataset. (\label{table:readability-metrics} \cite{yam2024babymodelsreadexploring}
)}
\end{table}

\textbf{Training Details:}  
All models are trained over 3 different context window sizes, namely 64, 128 and 256 tokens for 25,000 iterations using identical hyperparameters (learning rate, batch size, optimizer settings, etc.) to ensure fair comparisons, namely the values as used in nanoGPT \footnote{\href{https://github.com/karpathy/nanoGPT/blob/master/train.py}{https://github.com/karpathy/nanoGPT/blob/master/train.py}} (batch size = 12, initial learning rate = 6e-4, weight decay = 1e-1, gradient clip = 1.0, dropout = 0.0, optimizer = AdamW). To evaluate extrapolation, we test models on sequences longer than the training context, up to 16,384 tokens in multiples of 2. All experiments were conducted using an AMD EPYC 9554 CPU (128 cores), 252 GB RAM, and an NVIDIA H200 NVL GPU (140 GB), running Linux kernel 6.11 and Python 3.12.

\textbf{Metrics:}  
We report perplexity on the validation sets and further analyze the attention distributions by measuring the Shannon entropy of the softmax outputs.

\subsection{Experiments}

We evaluate the baseline and proposed approach on the perplexity metric, over three ranges, namely validation split of TinyStories featuring short-range dependency as well as on the 0-5k split and 5k-10k split of LTS which captures performance when medium-range and long-range dependencies are needed. In addition to pretraining experiments, we also demonstrate the performance enhancement obtained by using APE during fine-tuning.

\subsection{Results}

\subsubsection{Short-Range Dependency Handling}

As can be seen from Figure \ref{fig:perplexity_plots}, even though the perplexity metrics for RoPE and ALiBi begin to explode, soon after evaluation on prompts of length higher than the context window used during pretraining was encountered, APE still maintains an reasonably low perplexity value even for prompts of length 16,384 (which if, 64 context window training is considered, is then a 256-fold increase).

\begin{figure}[ht]
    \centering
    \includegraphics[width=0.32\textwidth]{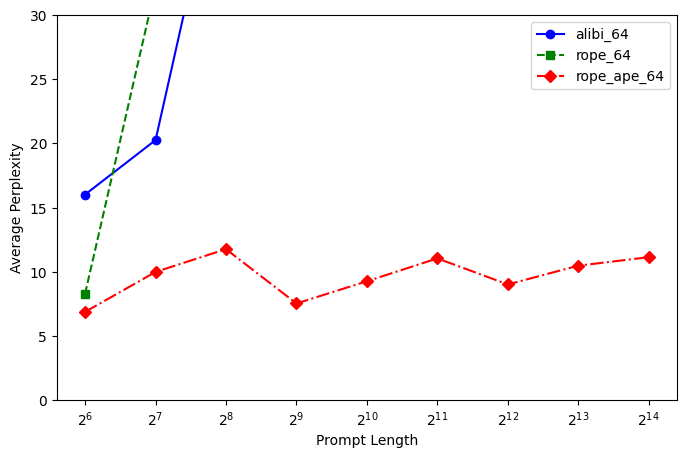}
    \includegraphics[width=0.32\textwidth]{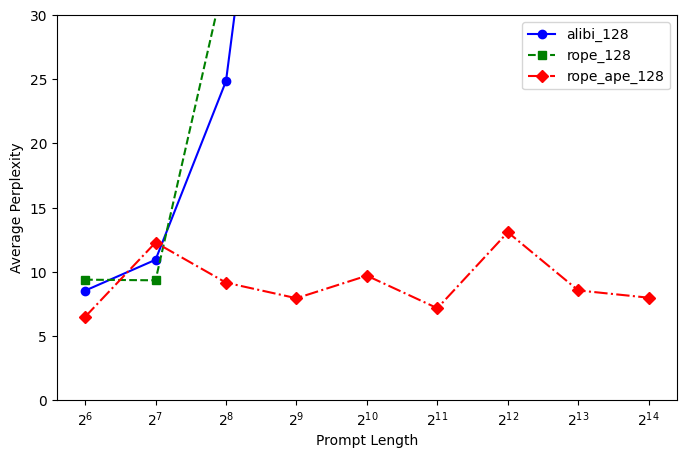}
    \includegraphics[width=0.32\textwidth]{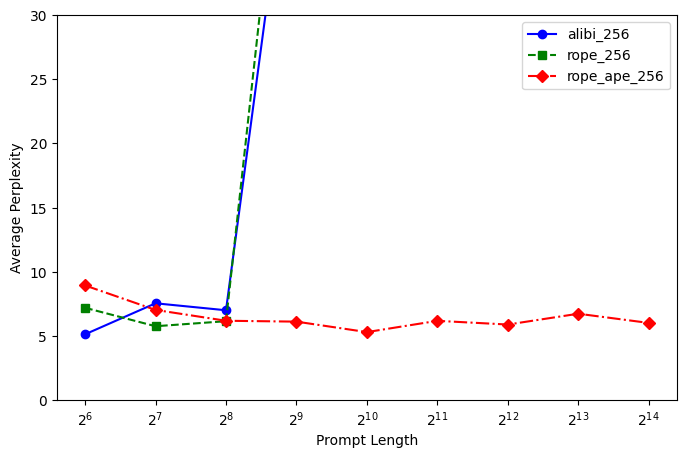}
    \caption{Plots of Perplexity on Validation set of Tiny Stories Dataset vs Prompt lengths for RoPE, ALiBi and APE trained with context windows of 64, 128 and 256.}
    \label{fig:perplexity_plots}
\end{figure}

We now analyze the internal behavior of models by examining the entropy of the attention distributions. As can be seen from Figure \ref{fig:attention_entropy}, in case of RoPE, it grows rapidly with increasing prompt length, as expected from our earlier analysis while for ALiBi, it soon starts saturating. APE on the other hand, achieves a middle ground, thereby enjoying the advantages of both the approaches.
\begin{figure}[ht]
\includegraphics[width = 0.32\textwidth]{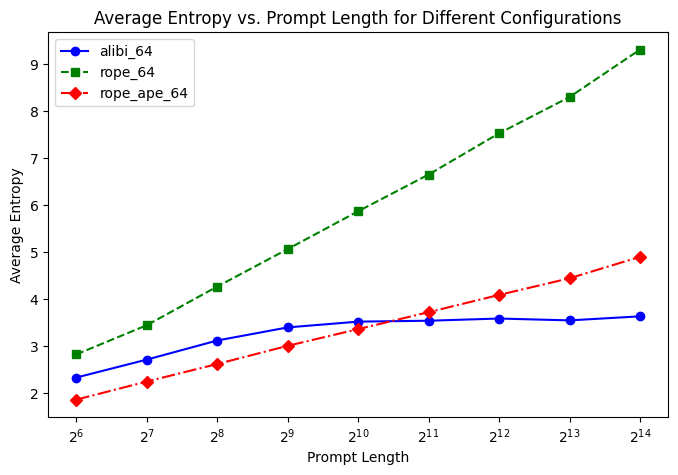}
\includegraphics[width = 0.32\textwidth]{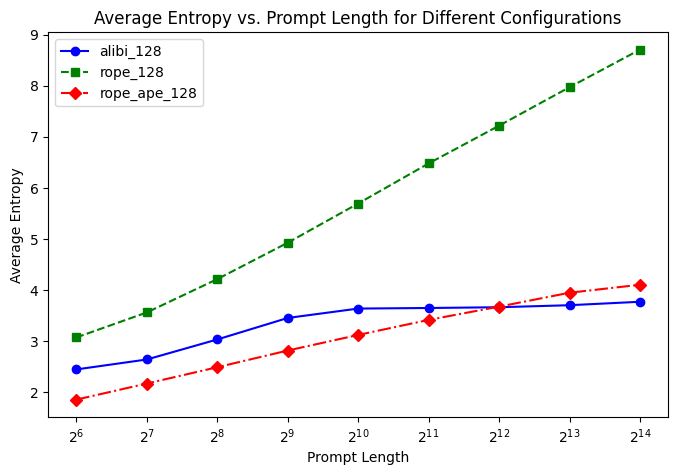}
\includegraphics[width = 0.32\textwidth]{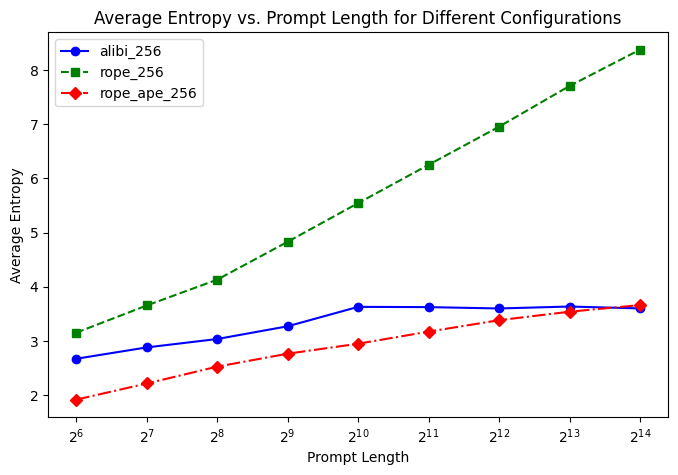} \\
\caption{Plots of Attention Entropy vs Prompt lengths for RoPE, ALiBi and APE trained with context windows of 64, 128 and 256}
\label{fig:attention_entropy}
\end{figure}

We then analyze in Figure \ref{figure:speed_memory} training speed, inference speed and memory footprints across the three approaches. As expected, due to extra learnable parameters in APE, its memory footprint is slightly higher. Moreover, ALiBi is superior when training and inference speeds are measured in terms of words per second. This might be attributed to the fact that, ALiBi simply adds a precomputed linear bias based on the relative distance in contrast to the costly trigonometric and matrix rotation computations required by RoPE and thus APE as well. 

\begin{figure}[ht]
\includegraphics[width = 0.32\textwidth]{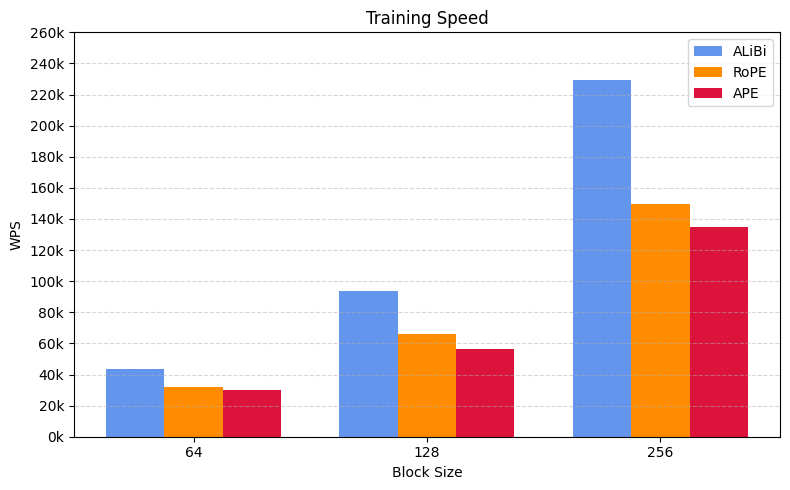}
\includegraphics[width = 0.32\textwidth]{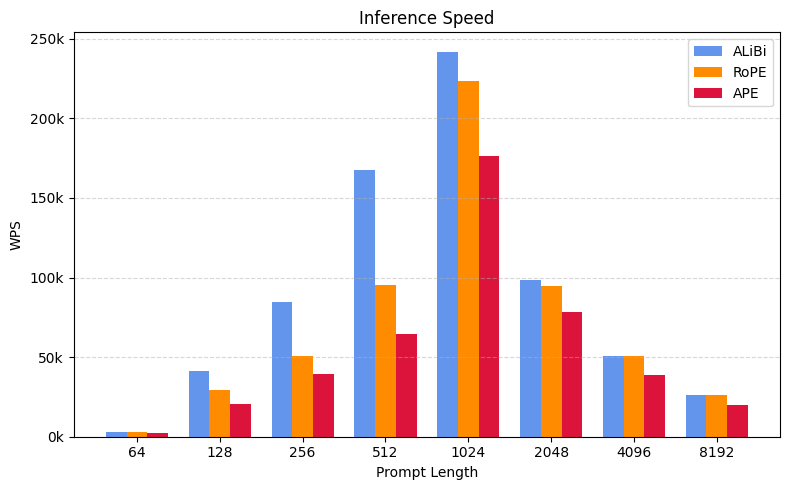}
\includegraphics[width = 0.32\textwidth]{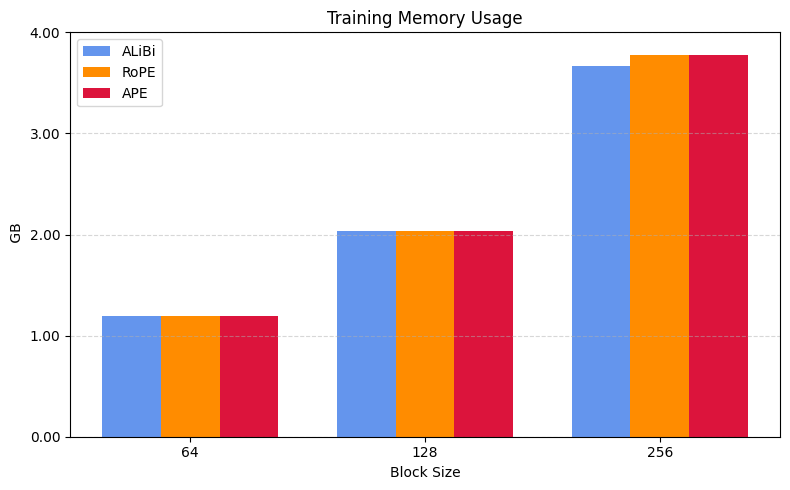} \\
\caption{A comparison of batched training, inference speed and memory usage of RoPE, ALiBi and APE.}
\label{figure:speed_memory}
\end{figure}

Despite this, APE can still be considered to memory efficient because from Figure \ref{figure:ALL}, performance of APE trained with context window of 64, is still remarkably better when compared with that of ALiBi and RoPE trained with context windows of 128 and 256. This implies that APE, by using $\sim 66 \%$ lower memory than ALiBi and RoPE, can still achieve remarkably superior performance.  

\begin{figure}[ht]
\centering
\includegraphics[width = 0.4\textwidth]{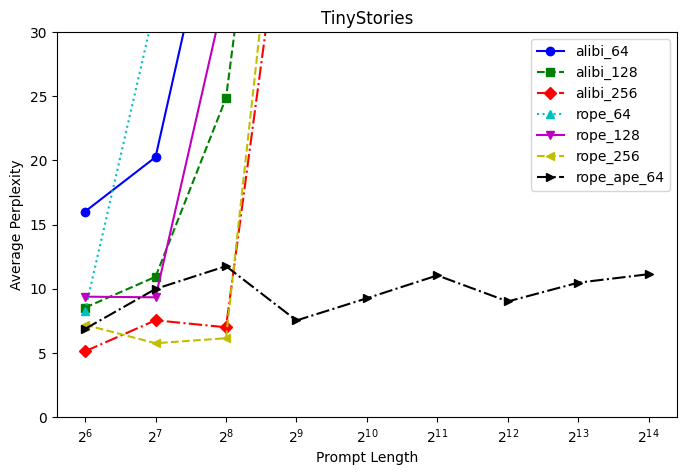}
\caption{Perplexity vs Prompt lengths for RoPE and ALiBi trained with context windows of 64, 128 and 256 compared with APE trained only with context window of 64}
\label{figure:ALL}
\end{figure}

\subsubsection{Medium-Range Dependency Handling}
We now, use the stories in the synthetic LTS dataset in the range 0 to 5000 words for this purpose. It can be seen that, while APE continues to fair consistently better than ALiBi and RoPE, an increase in perplexity is observed compared to that while testing upon on Tinystores. We believe that this can be attributed to the enhanced complexity of LTS over TinyStories, as is reflected when evaluated upon three metrics - namely,  Flesch reading ease (FRE), ARI (Automated Readability Index), and the Gunning fog index \ref{table:readability-metrics}.

\begin{figure}[ht]
\includegraphics[width = 0.32\textwidth]{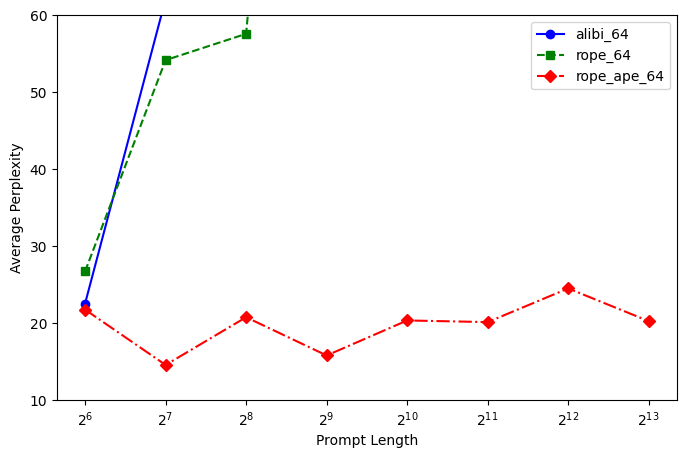}
\includegraphics[width = 0.32\textwidth]{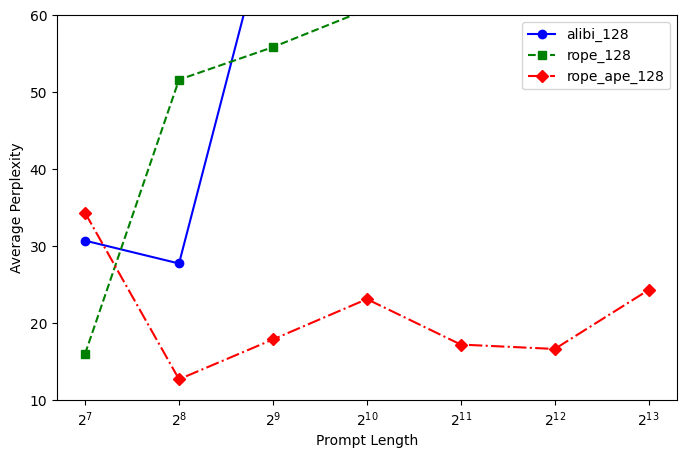}
\includegraphics[width = 0.32\textwidth]{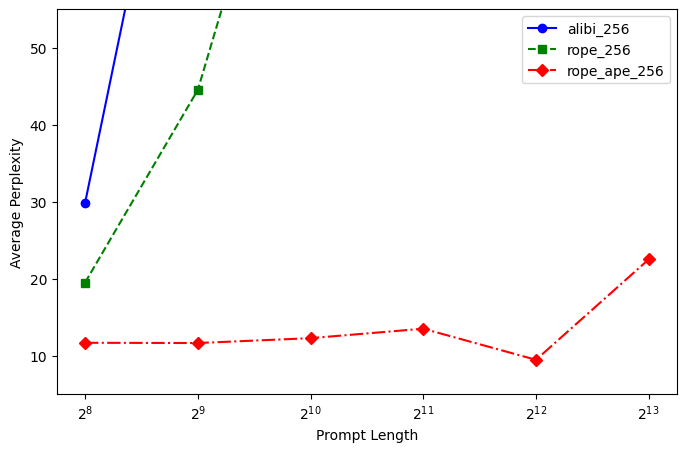} \\
\caption{Plots of Perplexity on LTS (upto 5,000 words) vs Prompt lengths for RoPE, ALiBi and APE trained with context windows of 64, 128 and 256}
\end{figure}

\subsubsection{Long-Range Dependency Handling}
 Even in the long-range scenario, as well similar trends are observed, demonstrating the superiority of APE in all regimes of prompt lengths over ALiBi and RoPE.
\begin{figure}[ht]
\includegraphics[width = 0.32\textwidth]{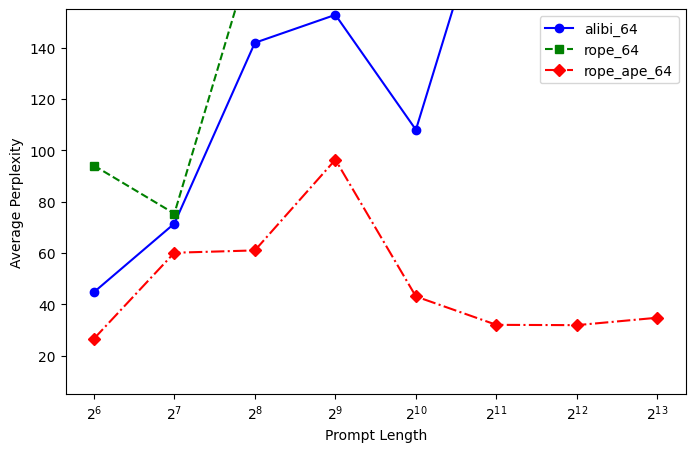}
\includegraphics[width = 0.32\textwidth]{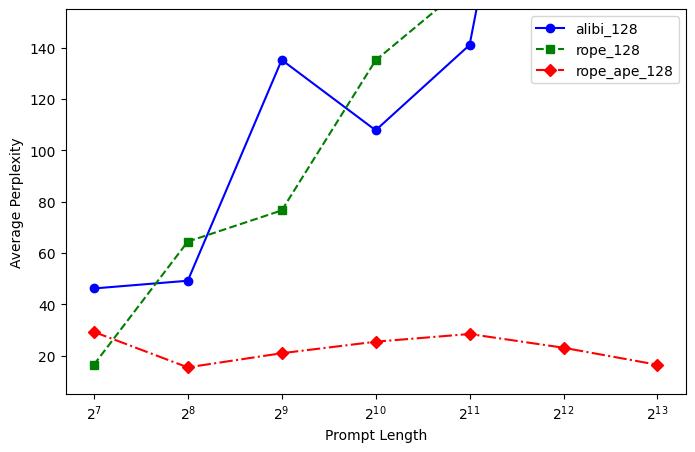}
\includegraphics[width = 0.32\textwidth]{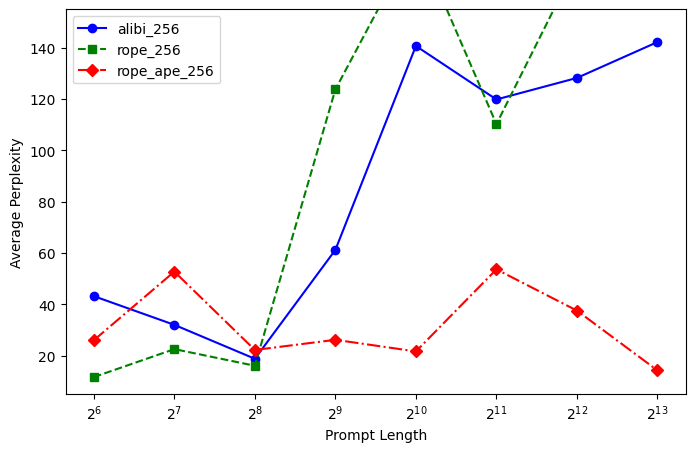} \\
\caption{Perplexity plots on LTS (from 5,000 to 10,000 words) vs Prompt lengths for RoPE, ALiBi and APE trained with context windows of 64, 128 and 256}
\end{figure}

\subsubsection{Fine-tuning}
We now, see how reliable of an approach is APE to deploy during fine-tuning. After training with RoPE and ALiBi, we finetuned the model with APE by using \textit{just 1\% of initial training corpora} and for \textit{only 500 iterations}. The resulting performance was considerably improved as seen in Figure \ref{Figure : Fine-tuning}

\begin{figure}[ht]
\centering
\includegraphics[width = 0.32\textwidth]{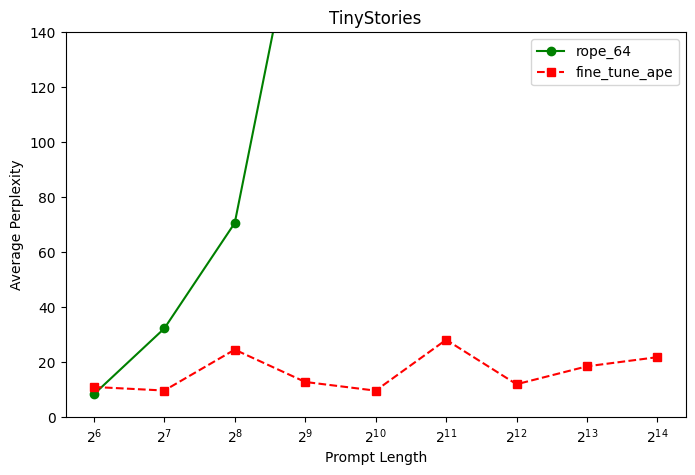}
\includegraphics[width = 0.32\textwidth]{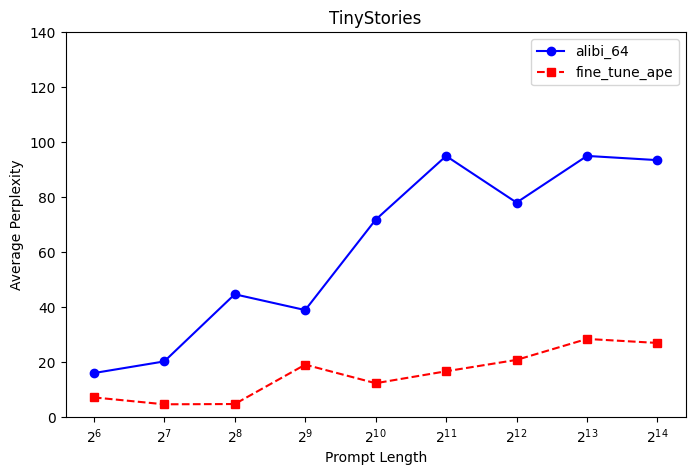} \\
\caption{Perplexity plots on TinyStories vs Prompt lengths while training on context window of 64 (a) RoPE, RoPE finetuned APE and (b) ALiBi, ALiBi finetuned with APE}
\label{Figure : Fine-tuning}
\end{figure}

\section{Limitations}
Despite the promising results, our work has several limitations. To begin with, our approach only holds in the context of positional encodings, which although has a significant impact on length generalization \cite{kazemnejad2023impact}, represents only one of the efforts in a race to achieving infinite length extrapolation. Secondly, only  four key properties - convergent normalization, entropy boundedness, long distance correlation preservation and gradient positional sensitivity - have been identified as desirable, based on formal intuition and proven mathematical results. There may exist more such properties which can lead either interesting equivalences or trade-offs, which haven't been explored in sufficient detail. Thirdly, our case study involves one such motivated approach by our framework and conducts experiments in a low-dimensional setting. Moreover, even though this particular approach has demonstrated that successful unification of two broad families of positional embeddings empirically, the trade-off between convergent normalization and LDCP is yet to be studied precisely. Also, though our model might be memory efficient,  when considered on the basis of context window, its memory footprint, throwput and inference is still lower than that of ALiBi. Lastly, the gradient positional sensitivity property might need modification so as to ensure it has some theoretical guarantee, despite empirical evidence supporting it (\cite{morita2024positional}  reports an unintuitive finding that positional encoding enhances learning of recurrent neural networks by stabilizing the gradients.  

\section{Future Research Directions}

Given that the field of exploration of positional encodings for extrapolation is still an active research area (\cite{chen2024hopenovelpositionalencoding}, \cite{hua2025fourierpositionembeddingenhancing}), we hope that our work might serve hereby as an intuitive guideline during ideation phases as to what properties should be looked out for in such embeddings. Failure to satisfy any one such property might indicate potential flaws before conducting experimental results. Secondly, although the case study is in a low-dimensional setting, but it obeys the Chinchilla scaling laws. The methods motivated by this framework may thus be able to generalize to larger models as well. Hence, scaling is one further such avenue. Lastly, since most of the widely used language modelling benchmarks such WikiText, Project Gutenberg-19 are much complex than TinyStories, testing of approaches would be entail much bigger models. However, now due to introduction of the LongTinyShort Stories dataset, length extrapolation capabilities can now be tested upon with small language models as well. 
 
\section{Conclusion}

In this paper, we introduced a unified framework for positional encoding that decomposes attention scores into adaptive multiplicative and additive components. Through theoretical analysis and extensive empirical evaluation on the Tiny Stories and Long Tiny Stories datasets, we demonstrated that our proposed Adaptive Positional Encoding (APE) outperforms existing methods in terms of infinite-context extrapolation, perplexity, and attention entropy. Despite some challenges in hyperparameter tuning and computational overhead, AEPE provides a promising direction for extending the capabilities of transformer-based language models. Future research will focus on optimizing these adaptive mechanisms, integrating our approach with scalable architectures, and exploring broader applications across diverse domains.

\newpage 

\bibliography{colm2025_conference}
\bibliographystyle{colm2025_conference}

\newpage

\appendix

\section{Proofs}

\subsection{}

\begin{theorem}[Equivalence of Entropy Boundedness and Convergent Normalization].
\begin{itemize}
  \item[(i)] If $Z<\infty$, then $H < \infty$.
  \item[(ii)] Conversely, if for every $L$ the truncated distribution
  $
  p_L(n)={e^{A(n)}}/{\sum_{k=0}^L e^{A(k)}}
  $
  has entropy uniformly bounded by $H_{\max}$ (i.e. $H_L \le H_{\max}$ for all $L$), then $Z<\infty$.
\end{itemize}
\end{theorem}

\begin{proof}  \label{ref:proof_entropy_equivalence}
\textbf{(i) \(\Rightarrow\) (ii):} Assume that \( Z < \infty \). Then \( p(n) = e^{A(n)}/Z \) defines a probability mass function on \( \mathbb{Z} \). To prove that \( H < \infty \), we argue as follows. Since \( Z < \infty \), for large \( |n| \) there exists an \(\epsilon > 0\) and a constant \( C > 0 \) such that
$
e^{A(n)} \le {C}/{|n|^{1+\epsilon}}.
$
This is because $\sum_{n=1}^{\infty} (1/n^{1+\epsilon})$ converges by the integral test.
Taking logarithms (for \( |n| \) sufficiently large) gives
$
A(n) \le - (1+\epsilon) \log|n| + \log C.
$
Now, the Shannon entropy is:  
$
H = -\sum_{n \in \mathbb{Z}} p(n) \log p(n) = -\sum_{n \in \mathbb{Z}} p(n) \left(A(n) - \log Z\right) = \log Z - \mathbb{E}_p[A(n)].
$ The expectation is:  
   \[
   \mathbb{E}_p[A(n)] = \sum_{n \in \mathbb{Z}} \underbrace{\frac{e^{A(n)}}{Z}}_{p(n)} A(n).
   \]  
   Substituting \( A(n) \leq -\epsilon \log |n| \), the dominant terms for large \( |n| \) are:  
   $
   e^{A(n)} A(n) \approx {-(\epsilon \log |n|)}/{|n|^{\epsilon}}.
   $  
   The series \( \sum_{n=1}^\infty {\log n}/{n^{\epsilon}} \) converges for \( \epsilon > 0 \). Hence, \( \mathbb{E}_p[A(n)] \) converges absolutely. Since \( \mathbb{E}_p[A(n)] \) and \( \log Z \) are finite, \( H = \log Z - \mathbb{E}_p[A(n)] \) is also finite.

\medskip

\textbf{(ii) \(\Rightarrow\) (i):} Assume that the truncated entropies 
$
H_L = -\sum_{n=-L}^{L} p_L(n) \log p_L(n)
$
are uniformly bounded by some constant \( H_{\max} \) for all \( L \), where 
$
p_L(n)= {e^{A(n)}}/{Z_L} \quad \text{and} \quad Z_L=\sum_{n=-L}^{L} e^{A(n)}.
$
Suppose, for contradiction, that \( Z = \lim_{L\to\infty} Z_L = \infty \). Then, for large \( L \) if the tail of \( e^{A(n)} \) does not decay sufficiently fast, the normalized distribution \( p_L(n) \) tends toward a near-uniform distribution over \( 2L+1 \) points. In that case, one would have
$
H_L \approx \log(2L+1),
$ which diverges as \( L\to\infty \). This contradicts the uniform bound \( H_L \le H_{\max} \). Therefore, \( Z \) must be finite. To summarize, the uniform entropy bound forces $E^{A(n)}$  to decay sufficiently fast to prevent $Z$
Z from diverging. This excludes distributions with heavy tails (e.g., uniform-like).

\end{proof}

\subsection{}

\begin{proposition}
    For any query and key vectors with bounded norms, RoPE violates Convergent Normalization and Entropy Boundedness.
\end{proposition}

\begin{proof} \label{ref:proof_RoPE_limitations}
Since \(\mathbf{R}(n)\) is an orthogonal matrix for every \(n\), it preserves the norm and the structure of the inner product. Write the vectors in a coordinate system adapted to the block-diagonal structure:
$
q_i = \bigl(q_i^{(1)}, q_i^{(2)},\dots,q_i^{(d/2)}\bigr), \quad k_j = \bigl(k_j^{(1)}, k_j^{(2)},\dots,k_j^{(d/2)}\bigr),
$
with each \(q_i^{(m)}, k_j^{(m)} \in \mathbb{R}^2\). In each block, the rotation \(\mathbf{R}^{(m)}(n)\) rotates \(k_j^{(m)}\) by an angle \(n\theta_m\), where $\mathbf{R}(n) = \text{diag}(\mathbf{R}^{(1)}(n), \mathbf{R}^{(2)}(n), \dots, \mathbf{R}^{(d/2)}(n))$. Hence,
\[
q_i^{(m)\top} \mathbf{R}^{(m)}(n) k_j^{(m)} = \|q_i^{(m)}\|\|k_j^{(m)}\|\cos\Bigl(n\theta_m - \varphi_m\Bigr),
\]
where \(\varphi_m\) is the phase difference between \(q_i^{(m)}\) and \(k_j^{(m)}\). Summing over all blocks gives
\[
g(n) = \sum_{m=1}^{d/2} \|q_i^{(m)}\|\|k_j^{(m)}\|\cos\Bigl(n\theta_m - \varphi_m\Bigr).
\]
Each cosine term is periodic and oscillatory, and unless all \(\theta_m\) are chosen to induce cancellation, the sum does not vanish as \(|n|\to \infty\). In particular, since no decay factor appears in \(g(n)\), the magnitude \(|q_i^\top \mathbf{R}(n) k_j|\) remains on the order of \(\mathcal{O}(1)\) for all \(n\). In other words, as \(n\) increases, the cosine factors oscillate but do not decay:
$
\bigl|q_i^\top \mathbf{R}(n) k_j\bigr| \sim \mathcal{O}(1).
$. This is precisely the phenomenon termed \emph{attention leakage}. Thus, the modified attention score \(A_{RoPE}(n) = q_i^\top \mathbf{R}(n) k_j\) remains bounded (and oscillatory) rather than decaying as \(|n|\) grows. Because there is no decay in \(A_{RoPE}(n)\), the exponential term \(e^{A_{RoPE}(n)}\) does not vanish for large \(|n|\). More formally, there exists some constant \(c>0\) such that for infinitely many \(n\), $
e^{A(n)} \ge e^{-c} > 0.
$. Therefore, the normalization constant
$Z = \sum_{n\in\mathbb{Z}} e^{A_{RoPE}(n)}
$ contains infinitely many terms that are bounded away from zero. This sum necessarily diverges: $
\lim_{L\to\infty} \sum_{n=0}^{L} e^{A_{RoPE}(n)} = \infty,
$ which violates the property of \emph{Convergent Normalization}. By the equivalence established in Theorem~\ref{thm:entropy_extrapolation}, the divergence of \(Z\) implies that the Shannon entropy of the attention distribution is unbounded. Hence, RoPE also violates \emph{Entropy Boundedness}.
\end{proof}

\subsection{}

\begin{proposition}
ALiBi fails to satisfy LDCP under assumption \ref{ref:assumption} and GPS.
\end{proposition}

\begin{proof} \label{ref:proof_ALiBi_limitation}
For ALiBi,
$
\mathbb{E}[A_{\text{ALiBi}}(n)]=\mathbb{E}[q^\top k]-m|n|=-m|n|.
$ Since by assumption, \(\mathbb{E}[q^\top k]=0\), we have \(\mathbb{E}[A_{\text{ALiBi}}(n)]=-m|n|\). Therefore, as \(|n|\) increases, the expectation decreases without bound, suppressing long-range contributions. Note that the additive term \(-m|n|\) does not depend on \(q\); hence, its derivative with respect to \(q\) is zero. Therefore,
$
{\partial A_{\text{ALiBi}}(n)}/{\partial q} = {\partial (q^\top k)}/{\partial q} = k.
$
Taking the norm, we obtain
$
\|{\partial A_{\text{ALiBi}}(n)}/{\partial q}\| = \|k\|.
$
This expression is independent of \(n\), which shows that the gradient with respect to \(q\) does not carry any information about the relative position \(n\). Although the decay term \(-m|n|\) ensures that \(A_{\text{ALiBi}}(n)\) decays (thereby leading to convergence of the normalization constant), the absence of \(n\)-dependent modulation in the gradient means that ALiBi lacks gradient positional sensitivity. This decoupling can hinder the model's ability to learn position-specific adaptations during training.
\end{proof}

\subsection{}

\begin{proposition}
RoPE satisfies LDCP under assumption \ref{ref:assumption} and GPS.
\end{proposition}

\begin{proof} \label{ref:proof_RoPE_Advantage}

Under the assumption \ref{ref:assumption}, 
$
\mathbb{E}[A_{\text{RoPE}}(n)]=\mathbb{E}[q^\top \mathbf{R}(n)k]=0,
$
and
$
\text{Var}[A_{\text{RoPE}}(n)]=\text{Var}[q^\top k]=1/d=\mathcal{O}(1).
$
Thus, by Chebyshev's inequality, for any \(\epsilon>0\),
\[
\mathbb{P}\Bigl(|A_{\text{RoPE}}(n)|<\epsilon\Bigr)\le \frac{1/d}{\epsilon^2}.
\]
For any fixed small \(\epsilon\) (e.g., \(\epsilon<1/\sqrt{d}\)), this probability is small, so with high probability \( |A_{\text{RoPE}}(n)|\ge \epsilon \). Hence, LDCP holds (in a probabilistic sense).
\end{proof}

\subsection{}

\begin{proposition}
For any query and key vector with bounded norms, ALiBi is able to achieve convergent normalization and entropy boundedness.
\end{proposition}

\begin{proof} \label{ref:proof_ALiBi_Advantage}
The additive bias \( b(n) = -m|n| \) leads to 
$
e^{A_{ALiBi}(n)} = e^{q_i^\top k_j} \cdot e^{-m|n|}.
$
Since \(e^{q_i^\top k_j}\) can be bounded by a constant (for bounded \(q_i\) and \(k_j\)), the softmax denominator is dominated by terms of the form
$
e^{A(n)} \sim e^{-m|n|},
$, which yields
\[
\sum_{n=0}^{\infty} e^{-m|n|} = \frac{1}{1-e^{-m}} < \infty.
\]
Thus ALiBi satisfies convergent normalization and hence entropy boundedness as well by Theorem \ref{thm:entropy_extrapolation}
\end{proof}

\subsection{}

\begin{proposition}
For any query and key vector with bounded norms, APE is able to achieve convergent normalization, entropy boundedness and gradient positional sensitivity. 
\end{proposition}

\begin{proof} \label{ref:proof_APE_advantage}
For sufficiently large \(|n|\) under the assumption of bounded query and key vectors, the adaptive rotation term 
$
q_i^\top \mathbf{R}_{\alpha(n)}(n) k_j 
$ can be bounded by some $C_0$. Hence, the multiplicative term satisfies
$
|\text{temp}(n) \cdot q_i^\top \mathbf{R}_{\alpha(n)}(n) k_j| \le {C_0}/({1+\lambda|n|}).
$
Thus, using the definition of \(A(n)\) we have
\[
e^{A(n)} \le \frac{C_0}{1+\lambda|n|} \cdot e^{b(n)}.
\]
Substituting the bias term, $
e^{b(n)} = e^{-\delta|n| - \beta \log(1+|n|) - \gamma \sqrt{|n|}} 
= e^{-\delta|n|}\,(1+|n|)^{-\beta}\, e^{-\gamma\sqrt{|n|}},
$
we obtain
\[
e^{A(n)} \le \frac{C_0\, e^{-\delta|n|}\, e^{-\gamma\sqrt{|n|}}}{(1+\lambda|n|)(1+|n|)^{\beta}}.
\]
For large \(|n|\), even though the denominator grows polynomially in \(|n|\), the exponential terms \(e^{-\delta|n|}\) and \(e^{-\gamma\sqrt{|n|}}\) ensure rapid decay. Therefore, by comparison with a convergent geometric series, the sum
$
\sum_{n\in\mathbb{\mathbb{N} \cup \{0\} }} e^{A_{APE}(n)}
$ converges provided \(\delta>0\) and \(\gamma>0\). BY virtue of Theorem \ref{thm:entropy_extrapolation}, it satisfies entropy boundedness as well and using a similar argument as in RoPE, APE achieves GPS (the dependency is also implied through temp($n$) as well in addition to rotation).
\end{proof}

\end{document}